\documentclass[11pt]{article}

\usepackage{graphicx}
\usepackage{geometry}
\usepackage{amsmath, amssymb} 
\usepackage{amsthm} 
\usepackage{mathtools}
\usepackage{cite}
\usepackage{xcolor} 
\usepackage{epigraph}
\usepackage{combelow} 

\newtheorem{theorem}{Theorem}

\newtheorem{lemma}[theorem]{Lemma}

\theoremstyle{definition}

\renewcommand{\subset}{\subseteq}

\renewcommand{\hat}{\widehat}
\renewcommand{\epsilon}{\varepsilon}

\def\<{\langle}
\def\>{\rangle}
\def\({\Big(}
\def\){\Big)}

\def\M{\mathcal{M}}

\def\calQ{\mathcal{Q}}

\def\Pavg{P_{\text{avg}}}
\def\Pmax{P_{\text{max}}}
\def\R{\mathbb{R}}

\begin{document}
\title{Maximal function pooling with applications}
\author{Wojciech Czaja\thanks{Norbert Wiener Center, Department of Mathematics, University of Maryland College Park. E-mail: wojtek@math.umd.edu} \and Weilin Li\thanks{Courant Institute of Mathematical Sciences, New York University. E-mail: weilinli@cims.nyu.edu} \and Yiran Li\thanks{Norbert Wiener Center, Department of Mathematics, University of Maryland College Park. E-mail: yiranli2019@gmail.com} \and Mike Pekala\thanks{Norbert Wiener Center, Department of Mathematics, University of Maryland College Park. E-mail: mpekala@umd.edu}}

\maketitle

\begin{abstract}
	Inspired by the Hardy–Littlewood maximal function, we propose a novel pooling strategy which is called maxfun pooling. It is presented both as a viable alternative to some of the most popular pooling functions, such as max pooling and average pooling, and as a way of interpolating between these two algorithms. We demonstrate the features of maxfun pooling with two applications: first in the context of convolutional sparse coding, and then for image classification.
\end{abstract}

\epigraph{This paper is dedicated to our friend, Professor John Benedetto, on the occasion of his 80th birthday.}{}

\section{Introduction} \label{sec:introduction}
In the last decade, the rapid developments in machine learning and artificial intelligence have captured the imagination of many scientists, across a full spectrum of disciplines. Mathematics has not been immune to this phenomenon. In fact, quite the opposite has happened and many mathematicians have been at the forefront of this fundamental research effort. Their contributions ranged from statistics, to optimization, to approximation theory, and last but not least, to harmonic analysis and related representation theory.
It is this last aspect that we want to focus on in this paper, as it has lead to many intriguing developments associated with the general theory of deep learning, and more specifically, with convolutional neural networks, see \cite{mallat2016}. 

Convolutional neural nets (CNNs) are a popular type of architecture in deep learning, which has shown an outstanding performance in various applications, e.g., in image and video recognition, in image classification \cite{ciregan2012multi}, \cite{ciresan2011flexible}, or in natural speech processing \cite{collobert}. CNNs can be effectively modeled by multiscale contractions with wavelet-like operations, applied interchangeably with pointwise nonlinearities \cite{mallat2016}. This results in a wealth of network parameters, which can negatively impact the numerical performance of the network. Thus, a form of dimensionality reduction or data compression is needed in order to efficiently process the information through the artificial neural network.
For these purposes many examples of CNNs use \emph{pooling} as a type of layer in their networks. Pooling is a dimension reduction technique that divides the input data into subregions and returns only one value as the representative of each subregion. Many examples of such compression strategies have been proposed to-date. Max pooling and average pooling are the two most widely used traditional pooling strategies and they have demonstrated good performance in application tasks \cite{goodfellow2016deep}. In addition to controlling the computational cost associated with using the network, pooling also helps to reduce overfitting of the training data \cite{graham}, which is a common problem in many applications.

In addition to the classical examples of maximal and average pooling, many more pooling methods have been proposed and were implemented in neural net architectures. Among those constructions, a significant role has been played by ideas from harmonic analysis, due to their role in providing effective models for data compression and dimension reduction. 
Spectral pooling was proposed in \cite{spectral} to perform the reduction step by truncating the representation in the frequency domain, rather than in the original coordinates. This approach claims to preserve more information per parameter than other pooling strategies and to increase flexibility for the size of pooling output.
Hartley pooling was introduced in \cite{Hartley} to address the loss of information that happens in the dimensionality reduction process. Inspired by the Fourier spectral pooling, the Hartley transform was proposed as the base, thus avoiding the use of complex arithmetic for frequency representations, while increasing the computational effectiveness and network's discriminability.
Transformation invariant pooling based on the discrete Fourier transform was introduced in \cite{tip} to achieve translation invariance and shape preservation thanks to the properties of the Fourier transform.
Wavelet pooling \cite{waveletp} is another alternative to the traditional pooling procedures. This method decomposes features into a two level decomposition and discards the first-level subbands to reduce the dimension, thus addressing the overfitting, while reducing features in a structurally conscious manner.
Multiple wavelet pooling \cite{mwaveletp} builds upon the wavelet pooling idea, while introducing more sophisticated wavelet transforms such as Coiflets and Daubechies wavelets, into the process.
An even more general approach was proposed in \cite{elpp}, where $\ell_p$ pooling was defined based on the concept of a representation in terms of general frames for $\R^d$, to provide invariance to the system.

In this paper we follow in the footsteps of the aforementioned constructions and propose a novel method for reducing the dimension in CNNs, which is based on a fundamental concept from harmonic analysis. Inspired by the \textit{Hardy--Littlewood maximal function}, \cite {HL}, cf., \cite{coifman1974weighted}, \cite{muckenhoupt1972weighted} for its modern treatment, we introduce a novel pooling strategy, called \textit{maxfun pooling}, which can be viewed as a natural alternative to both max pooling and average pooing. In particular, max pooling takes the maximum value in each pooling region as the scalar output, and average pooling takes the average of all entries in each pooling region as the scalar output. As such, maxfun pooling can be interpreted as a novel and creative way of interpolating between max and average pooling algorithms.

In what follows, we introduce a discrete and computationally feasible analogue of the Hardy-Littlewood maximal function. The resulting operator depends on two integer parameters $b$ and $s$, corresponding to the size of the pooling region and the stride. We limit the support of this operator to be finite, we discretize it and define the maximal function pooling operation, denoted throughout this paper by \textit{maxfun pooling}. The maxfun pooling computes averages of subregions of different sizes in each pooling region, and it selects the largest average among all. To demonstrate the features of maxfun pooling, we present two different applications. First, we study its properties in the realm of convolutional sparse coding. It has been shown that feedforward convolutional neural networks can be viewed as convolutional sparse coding \cite{papyan2016convolutional}. Moreover, under the point of view presented by the convolutional sparse coding, stable recovery of the signal contaminated with noise can be achieved, given simple sparsity conditions \cite{papyan2016convolutional}. Equivalently, it implies that feedforward neural networks maintain stability under noisy conditions. The case of pooling function analyzed via convolutional sparse coding is studied in \cite{kabkab2017spatial}, where the two common pooling functions, max pooling and average pooling are analyzed. We follow the framework presented in \cite{kabkab2017spatial} and we show that stability of the neural network under the presence of noise is also preserved with maxfun pooling. We close this paper with a different application, presenting illustrative numerical experiments utilizing maxfun pooling for image classification.

\section{Preliminaries}  \label{sec:prelim}

In this section, we elaborate upon the role of pooling in neural networks, and discuss two traditional strategies, max and average pooling. We focus on image data as our main application domain, but we mention that the subsequent definitions can be readily generalized. 

We view images as functions on a finite lattice $X\colon [M]\times [N]\to\R$, where $[K]:=\{0,1,\dots,K-1\}$, or $X\in\R^{M\times N}$ in short. Its $(i,j)-th$ coordinate is denoted $X_{i,j}$. In practice, it is convenient to fold images into vectors. Slightly abusing notation, for $X\in\R^{M\times N}$, we also let $X\in\R^{MN}$ denote its corresponding vectorization,
\begin{equation}
\label{eq:vectorization}
X
=(X_{1,1}, X_{2,1},\dots, X_{M,1}, X_{1,2},\dots, X_{M,2}, \dots, X_{M,N})^T.
\end{equation}
Throughout this chapter, we will not make distinctions between vectors and images.
As a consequence, we shall assume without any loss of generality that pooling layer's input is always nonnegative.

\subsection{Max and average pooling}

Both the maximum and average pooling operators depend on a collection of sets $\calQ\subset [M]\times[N]$, which we refer to {\it pooling regions}. There is a standard choice of $\calQ$, which we describe below. Fix an odd integer $s$ and set $m=\lfloor M/s\rfloor$ and $n=\lfloor N/s \rfloor$. For each pair of integers $(k,\ell)$, we define the square $Q_{k,\ell}$ of size $s\times s$ by
\begin{equation}
\label{eq:Q}
Q_{k,\ell }
:= \{(i,j)\colon (k-1)s \leq i < ks, \, (\ell -1)s \leq j < \ell s\}. 
\end{equation}
It is common to refer to $s$ as the {\it stride}. The stride determines the size of each pooling region $Q_{k,\ell}$ and the total number of squares in $\calQ$. The use of a stride implicitly reduces the input data dimension since an image of size $M\times N$ is then reduced to one of size $m\times n$.

For this collection $\calQ$, the {\it average pooling} operator $\Pavg\colon\R^{M\times N}\to\R^{m\times n}$ is given by
\[
(\Pavg X)_{k,\ell}:=\frac{1}{s^2} \sum_{(i,j)\in Q_{k,\ell}} X_{i,j}.
\]
For the same collection, the {\it max pooling} operator $\Pmax\colon \R^{M\times N}\to\R^{m\times n}$ is defined as
\[
(\Pmax X)_{k,\ell}:=\max_{(i,j)\in Q_{k,\ell}} X_{i,j}.
\]
In other words, $\Pavg$ simply averages the input image on each $Q_{k,\ell}$, while $\Pmax$ is the supremum of the values of the image restricted to $Q_{k,\ell}$.

\subsection{Maximal function}

For a locally integrable function $f \in L^1_\text{\rm loc}(\R^d)$, we can define its {\it Hardy--Littlewood maximal function} $M(f)$ as
\begin{equation}
\label{e.8.1}
\forall\; x \in \R^d, \quad M(f)(x) = \sup_{B(x)} \frac{1}{m^d(B(x))} \int_{B(x)} |f|\; dm^d,
\end{equation}
where the supremum is taken over all open balls $B(x)$ centered at $x$ and the integral is taken with respect to the Lebesgue measure $m^d$. Because of this last property, sometimes we talk about the {\it centered} maximal function, as opposed to the non-centered analogue. The function $M$ was introduced by Godfrey H. Hardy and John E. Littlewood in 1930 \cite{HL}. Maximal functions had a profound influence on the development of classical harmonic analysis in the 20th century, see, e.g., \cite{Grafakos} and \cite{Stein}. Among other things, they play a fundamental role in our understanding of the differentiability properties of functions, in the evaluation of singular integrals, and in applications of harmonic analysis to partial differential equations. One of the key tools in this theory is the following Hardy--Littlewood lemma.

\begin{theorem}
\label{theorem}
{\bf Hardy--Littlewood lemma}\\
Let $(\R^d, \M(\R^d), m^d)$ be the Lebesgue measure space. Then, for any $f \in L^1(\R^d)$,
\[
\forall\; \alpha>0, \quad m^d\left( \{ x\in \R^d: M(f)(x) >\alpha \} \right) 
\le \frac{3^d}{\alpha} \int_{\R^d}|f|\; dm^d.
\]
\end{theorem}
In view of Theorem \ref{theorem}, the maximal function is sometimes interpreted as encoding the worst possible behavior of the signal $f$. This point of view is further exploited through the {\it Calder\'on-Zygmund decomposition theorem}.

\begin{theorem}
\label{theorem2}
{\bf Calder\'on-Zygmund decomposition}\\
Let $(\R^d, \M(\R^d), m^d)$ be the Lebesgue measure space. Then, for any $f \in L^1(\R^d)$ and any $\alpha >0$, there exists $F, \Omega \subset \R^d$ such that $\R^d = F \cup \Omega$, $F \cap \Omega = \emptyset$, and
\begin{enumerate}
    \item $f(x) \le \alpha$, $m^d$-$a.e.$ in $F$
    \item $\Omega$ is the union of cubes $Q_k$, $k=1, \ldots$, whose interiors are pairwise disjoint, edges are parallel to the coordinate axes, and for each $k=1, \ldots$, we have
    \[
    \alpha < \frac{1}{m^d(Q_k)} \int_{Q_k} |f|\; dm^d \le 2^d \alpha.
    \]
\end{enumerate}
\end{theorem}
The Calder\'on-Zygmund decomposition allows us to split an arbitrary integrable function into its ``good" (i.e., small) and ``bad" (i.e., large) parts, which is a standard technique in addressing problems in the theory of Calder\'on-Zygmund operators. In this approach, the maximal function helps us control the undesirable behaviour of the signal by constraining the regions (represented by balls or cubes) on which this happens. It is this aspect of the maximal functions that we also aim to exploit in our construction of the {\it maxfun pooling}, which will be defined in the next section.

\section{Maxfun Pooling}  \label{sec:maxfun-def}

For maxfun pooling, we first fix an odd stride $s$ and parameter $b$ such that $2b+1\leq s$. Consider the same collection $\calQ$ with $Q_{k,\ell}$ defined in \eqref{eq:Q}, and let $q_{k,\ell}$ denote the center of $Q_{k,\ell}$, which is well defined since $s$ is assumed odd. For each $r\leq b$, we define the sub-squares,
\begin{equation}
\label{eq:suppB}
B_{k,\ell,r}
\subset Q_{k,\ell},
\end{equation}
where $B_{k,\ell,r}$ denotes the square of side length $2r+1$ whose center is also $q_{k,\ell}$. 

The {\it maxfun pooling} operator $M_{b,s}\colon \R^{M\times N}\to \R^{m\times n}$ is given by 
\begin{equation}
\label{eq:maxfun}
(M_{b,s} X)_{k,\ell}
:=\max_{1\leq r\leq b} \, \left( \frac{1}{(2r+1)^2} \sum_{(i,j)\in B_{k,\ell,r}} X_{i,j} \right). 
\end{equation}
The quantity $(M_{b,s} X)_{k,\ell}$ is the average of $X$ on $B_{k,\ell,r}$, and the optimal choice of $r$ depends on the restriction of $X$ to $Q_{k,\ell}$. 

\par 
Let us briefly justify the definition of maxfun pooling. Similar to the maximal function given in \eqref{e.8.1} where the supremum is taken over all balls $B$ centered at a fixed $x$, each $B_{k,\ell,r}$ is centered at $q_{k,\ell}$. Our convention regarding the nonnegative inputs of the pooling layer aligns with the use of the absolute value in the definition of the maximal function. Finally, the computational complexity of evaluating maxfun can be further reduced by choosing the $b$ parameter to be strictly smaller than $(s-1)/2$.

\par
For each region $Q_{k,\ell}$ and any image $X$, clearly we have the inequalities
\[
(\Pavg X)_{k,\ell}
\leq (M_{b,s}X)_{k,\ell}
\leq (\Pmax X)_{k,\ell}.
\]
From this point of view, the output of maxfun on each pooling region is quantitatively between that of average and max pooling. It is not difficult to construct examples of $X$ and $Y$ for which $\Pavg X =M_{b,s}X$ and $\Pmax Y= M_{b,s}Y$. 

Another main conceptual difference between these three pooling operators is the amount of information of $X$ that determines their outputs. Average pooling depends on the values of $X$ on the entire region $Q_{k,\ell}$, whereas max pooling selects the largest value of $X$ on $Q_{k,\ell}$, which is not necessarily representative of $X$. From this perspective, average pooling is more robust to perturbations. Maxfun pooling is adaptive in the sense that the optimal region $B_{k,\ell,r}$ depends on $X$, which is a different procedure for selecting a representative value.


\section{Convolutional Sparse Coding} \label{sec:sparse-coding}
The sparse coding problem is an important problem in signal processing, where one aims at finding a low dimensional representation using few atoms for high dimensional data \cite{papyan2016working}. Following standard conventions, we treat an image $X$ as a vector $X\in\R^N$ using the vectorization operator defined in \eqref{eq:vectorization}. Given a vector $X \in \mathbb{R}^N$, and a dictionary $D \in \mathbb{R}^{N\times M}$, the sparse coding problem attempts to find the sparsest vector $\Gamma \in \mathbb{R}^M$ such that $X= D\Gamma$. In other words, for a fixed dictionary $D\in \mathbb{R}^{N\times M}$, the sparse coding problem attempts to solve:
\begin{equation}
\min_{\Gamma}\|\Gamma\|_0 \quad s.t. \quad D\Gamma=X,
\label{sc1}
\end{equation}
where $\|\Gamma\|_0$ is the $l_0$ pseudo norm, which provides the number of non-zero elements in vector $\Gamma$.
Each column in $D$ represents one element in the dictionary, and finding the dictionary to represent data $X$, so that minimal number of dictionary elements are used, solves the sparse coding problem. 

Restriction on the sparsity of $\Gamma$ with respect to the mutual coherence of the dictionary $D$ can guarantee uniqueness of the solution to (\ref{sc1}). The mutual coherence of a matrix $D$, see \cite{donoho2006stable}, is defined as:
\begin{equation}
\mu(D)=\max_{i \neq j}\frac{|di^Td_j|}{\|d_i\|_2\cdot \|d_j\|_2},
\end{equation}
where $d_i$'s are the columns of matrix $D$.
However, finding the solution remains NP hard. Relaxation of the model to allow noise and form error bounds leads to the following formulation:
\begin{equation}
\min_{\Gamma} \|\Gamma\|_0 \quad s.t. \quad \|D\Gamma-X\|<\epsilon.
\end{equation}
When high dimensional signals are present, an alternative method called the convolutional sparse coding model (CSC) was proposed. One attempts to represent the whole signal $X \in \mathbb{R}^N$ as a multiplication of a global convolutional dictionary $D \in \mathbb{R}^{N\times Nm_1}$ and a sparse vector $\Gamma \in \mathbb{R}^{Nm_1}$. $D$ is constructed by shifting a local matrix of size $n_0 \times m_1$ in all possible positions, as shown in Figure \ref{Dconvolution}.
\begin{figure}
\centering
\includegraphics[width=\textwidth]{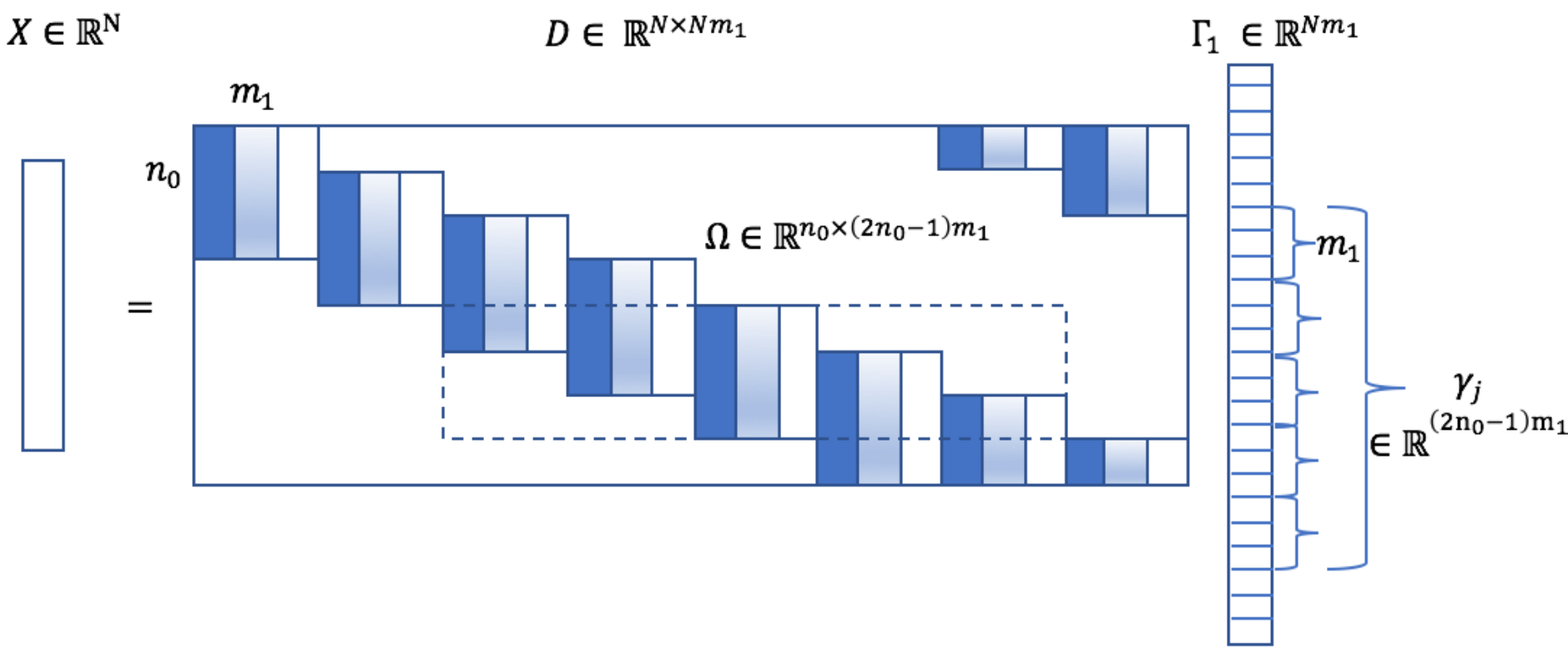}
\caption{Convolutional Sparse Coding, level 1 \cite{papyan2016convolutional}}
\label{Dconvolution}
\end{figure}
We define the $j$th stripe $\gamma_j$ of the sparse vector $\Gamma$ as a group of $2n_0-1$ adjacent sparse vectors of length $m_1$, starting at the $j$th vector of length $m_1$. See Figure \ref{Dconvolution} for an illustration.

The stripe $\gamma_j$ gives the representation of a patch of $X$, $x_j$ of length $n_0$ by $x_j=\Omega_j \gamma_j$. $\Omega_j \in \mathbb{R}^{n_0\times (2n_0-1)m_1}$ is a submatrix of $D$, called a stripe dictionary consisting of $n_0$ consecutive rows of $D$ and the columns of zeros removed. The $l_{0,\infty}$ norm of the global sparse vector $\Gamma_1$ is defined by the maximum number of non-zeros in any stripe of length $(2n_0-1)m_1$ extracted from it, i.e.,
\begin{equation}
\|\Gamma\|_{0,\infty}=\max_{i \in \{1, ..., N\}} \|\gamma_i\|_0.
\end{equation}
Here $\|\cdot\|_0$ is the zero norm that gives the number of nonzero elements of a vector. 

A multi-layer convolutional sparse coding model is defined so that the output sparse vector $\Gamma$ from the previous layer serves as the input vector in the next layer, and we aim at finding a new representation $\Gamma_2$ for a new set of dictionary $D_2$. Formally, the problem of finding solutions to multi-layer convolutional sparse coding problem is defined as the deep coding problem $DCP_{\lambda}$ in \cite{papyan2016convolutional}:
\begin{align}
\label{problem5}
\begin{aligned}
&Find \quad \{\Gamma_i\}_{i=1}^L \\
&X=D_1\Gamma_1, \\
&\Gamma_1 = D_2\Gamma_2, \\
&\vdots\\
&\Gamma_{L-1} = D_L\Gamma_L, \\
\end{aligned}
&&
\begin{aligned}
 &s.t.\\
 & \|\Gamma_1\|_{0,\infty} &\leq \lambda_1\\
 & \|\Gamma_2\|_{0,\infty} &\leq \lambda_2\\
 \\
 &\|\Gamma_L\|_{0,\infty} &\leq \lambda_L
\end{aligned}
\end{align}
where $\lambda_i$ are bounds on sparsity of the output vector $\Gamma_i$ at each level, and $L$ is the number of layers. Note that we want to find representations of the input vectors at each layer that are sparse in terms of its stripe sparsity, defined by $\|\cdot\|_{0,\infty}$.\par
In practice, the input signal $X$ can be contaminated with noise, and we have $Y=X+E$ as the input signal instead of $X$, where $E$ represents noise. In this case, we relax the constraint and allow the representation to vary within some error bounds of the input signal. The deep coding problem when noise is present ($DCP_{\lambda}^{\epsilon}$) \cite{papyan2016convolutional} is defined as:
\begin{align}
\label{DCSPe}
\begin{aligned}
&Find \quad \{\Gamma_i,P_i\}_{i=1}^L \\
&\|Y-D_1\Gamma_1\|_2\leq \epsilon_1, \\
&\|\Gamma_1-D_2\Gamma_2\|_2\leq \epsilon_2, \\
&\vdots\\
&\|\Gamma_{L-1}-D_L\Gamma_L\|_2\leq \epsilon_L, \\
\end{aligned}
&&
\begin{aligned}
&s.t.\\
 & \|\Gamma_1\|_{0,\infty} &\leq \lambda_1\\
 & \|\Gamma_2\|_{0,\infty} &\leq \lambda_2\\
 \\
 &\|\Gamma_L\|_{0,\infty} &\leq \lambda_L.
 \end{aligned}
\end{align}
Here $\epsilon_i$ is the error bound that are allowed in the $i$th layer. \par
Uniqueness of the solution to the $DCP_{\lambda}$ model, and the stability of the solution to the $DCP_{\lambda}^{\epsilon}$ problem have been shown in \cite{papyan2016convolutional}. The connection between deep convolutional sparse coding problem and feed forward neural network has also been demonstrated in \cite{papyan2016convolutional}, where it is proven that one can view the output vector $\Gamma_i$ from each layer of the $DCP_{\lambda}$ problem as the output from one layer of feed forward simplified convolutional neural network (CNN), and thus the deep convolutional sparse coding problem can be viewed as a signal reconstruction problem for simplified CNN models.

Pooling is a common operation included in CNNs that serves as feature extraction method to reduce redundancy of representation of signal and save computational resources. It has been shown that adding max pooling and average pooling in the feed forward path preserves the stability of the neural network \cite{kabkab2017spatial}. We demonstrate that the maxfun pooling, preserves the stability of a convolutional neural network when added in between layers of convolutions.

Given a input signal $X$, the deep convolutional sparse coding problem with pooling ($DCPP$) is defined by \cite{kabkab2017spatial}
\begin{align}
\label{DCPP}
\begin{aligned}
&Find \quad \{\Gamma_i,P_i\}_{i=1}^L \\
&X=D_1\Gamma_1, \\
&P_1 = D_2\Gamma_2, \\
&\vdots\\
&P_{L-1} = D_L\Gamma_L, \\
\end{aligned}
&& 
 \begin{aligned}
 &s.t.\\
& \|\Gamma_1\|_{0,\infty} &\leq \lambda_1\\
 & \|\Gamma_2\|_{0,\infty} &\leq \lambda_2\\
 \\
 &\|\Gamma_L\|_{0,\infty} &\leq \lambda_L,
 \end{aligned}
 &&
\begin{aligned}
\\
 &P_1=Pool_{b_1,s_1}(\Gamma_1),\\
 &P_2=Pool_{b_2,s_2}{}(\Gamma_2),\\ 
 &\vdots\\
 & P_L=Pool_{b_L,s_L}(\Gamma_L),
 \end{aligned}
\end{align}
where $Pool_{_ib,s}$ denotes the pooling operation at the step $i$. We take $Pool_{b_i,s}$ to be the maxfun pooling, i.e., $Pool_{b_i,s} =M_{b,s}$, as defined in (\ref{eq:maxfun}). 

Problem (\ref{DCPP}) intends to find a stable sparse representation $\Gamma_1$ of $X$ with dictionary elements in $D_1$, given restriction on the stripe-sparsity of $\Gamma_1$. Then pooling operation is performed on $\Gamma_1$ to get $P_1$. In second layer, we attempt to find the sparse representation $\Gamma_2$ of $P_1$ with dictionary elements in $D_2$. The stripe-sparsity of $\Gamma_2$ is restricted to be no greater than $\lambda_2$. We repeat the process $L$ times. 

If our input signal $X$ is contaminated by noise $E$, we are still interested in finding a sparse representation that is stable. Define the deep convolutional sparse coding problem with pooling when noise is present ($DCPP^{\epsilon}$) by \cite{kabkab2017spatial}
\begin{align}
\label{DCPPe}
\begin{aligned}
&Find \quad \{\Gamma_i,P_i\}_{i=1}^L \\
&\|Y-D_1\Gamma_1\| \leq \epsilon_1, \\
&\|P_1 - D_2\Gamma_2\| \leq \epsilon_2, \\
&\vdots\\
&\|P_{L-1} - D_L\Gamma_L\|\leq \epsilon_L, \\
\end{aligned}
&& 
 \begin{aligned}
 &s.t.\\
& \|\Gamma_1\|_{0,\infty} &\leq \lambda_1\\
 & \|\Gamma_2\|_{0,\infty} &\leq \lambda_2\\
 \\
 &\|\Gamma_L\|_{0,\infty} &\leq \lambda_L,
 \end{aligned}
 &&
\begin{aligned}
\\
 &P_1=Pool_{b_1,s_1}(\Gamma_1),\\
 &P_2=Pool_{b_2,s_2}{}(\Gamma_2),\\ 
 &\vdots\\
 & P_L=Pool_{b_L,s_L}(\Gamma_L),
 \end{aligned}
\end{align}
It has been shown in \cite{kabkab2017spatial} that when max pooling and average pooling are used, the stability of solution to the $DCPP^{\epsilon}$ problem is preserved. We show that when we use maxfun pooling, the stability result also holds. We prove the following theorem:
\begin{theorem}
\label{theorem1}
Suppose a vector $X$ satisfies the $DCPP$ model in (\ref{DCPP}), but is contaminated with noise $E$, where $\|E\|_2 \leq \epsilon$, resulting in $Y=X+E$. Suppose $\{\Gamma^*_i,P_i^*\}_{i=1}^L$ solves the problem in (\ref{DCPP}) and $\{\hat{\Gamma}_i,\hat{P}_i\}_{i=1}^L$ solves the problem in (\ref{DCPPe}). If 
\begin{equation}
\begin{split}
&\|\Gamma^*_i\|_{0,\infty} \leq \lambda_i <\frac{1}{2} \Big (1+\frac{1}{\mu(D_i)} \Big), \quad \forall 0 \leq i \leq L,\\
&\epsilon_0 = \epsilon, \quad \epsilon_i^2 =\frac{4\epsilon_{i-1}^2}{1-(2\|\Gamma_i^*\|_{0,\infty}-1)\mu(D_1)} \quad \forall i \geq 1,\\
&\text{then for all}\quad 1\leq i \leq L,\\
&\|P_i^*-\hat{P}_i\|^2_2 \leq \|\Gamma_i^*-\hat{\Gamma}_i\|_2^2 \leq \epsilon_i^2.
\end{split}
\label{assumption}
\end{equation}
Here maxfun pooling is used as the pooling operation and we assume that the minimum pooling region size $b \geq 1$.
\label{T1}
\end{theorem}
In order to prove Theorem \ref{T1}, we first prove the following Lemma for maxfun pooling.
\begin{lemma}
Let $X$ and $\hat{X}$ be two functions in $\mathbb{R}^{N\times N}$, and let $P = M_{b,s}(X)$, $\hat{P}=M_{b,s}(\hat{X})$ be the outcome of maxfun pooling of $X$ and $\hat{X}$,respectively, and assume that $s \geq b$. Then $\|P-\hat{P}\|_2 \leq \|X-\hat{X}\|_2$, where $\|\cdot\|_2$ denotes the Frobenius norm.
\label{lemma4.2}
\end{lemma}
\begin{proof}
Let $B_{k,\ell,r} $ be the set of indices that represents points in the centered sub-square of side length $2r + 1$ in the $(k,\ell)$-th pooling region. Let 
\[
\gamma_{k,\ell,r}=\frac{1}{(2r + 1)^2}\sum_{(i,j) \in B_{k,\ell,r}}X_{i,j}
\quad\text{and}\quad 
\hat{\gamma}_{k,\ell,r} =\frac{1}{(2r + 1)^2}\sum_{(i,j) \in B_{k,\ell,r}}\hat{X}_{i,j}.
\]
We take $r_{k,\ell}^*=\underset{1\leq r\leq b}{\text{argmax}}\gamma_{k,\ell,r}$ to indicate the $r$ parameter of the sub-square with the maximum of $\gamma_{k,\ell,r}$ over all $r$ for each $(k,\ell)$-th region; and we take $\hat{r}_{k,\ell}^*=\underset{1\leq r\leq b}{\text{argmax}}\hat{\gamma}_{k,\ell,r}$ to indicate the $r$ parameter of the sub-square with maximum of $\hat{\gamma}_{k,\ell,r}$ over all $r$ for each $(k,\ell)$-th region. 
We also take the $r_{min}^*$ to be minimum of all $r_{k,\ell}^*$, and $\hat{r}_{min}^*$ to be minimum of all $\hat{r}_{k,\ell}^*$. Furthermore, we let $K_1$ be the set of indices of $(k,\ell)$ so that $\gamma_{k,\ell,r_{k,\ell}^*} \geq \hat{\gamma}_{k,\ell,\hat{r}_{k,\ell}^*}$. Let $K_2$ be the set of indices of $(k,\ell)$ so that $\gamma_{k,\ell,r_{k,\ell}^*} < \hat{\gamma}_{k,\ell,\hat{r}_{k,\ell}^*}$. Then we have
\begin{align}
&\|P-\hat{P}\|^2_2 \nonumber \\
 &=\sum_{(k,l)} \left(\underset{1\leq r \leq b}{\max} \Big(\frac{1}{(2r + 1)^2}\sum_{(i,j) \in B_{k,\ell,r}}X_{i,j} \Big)- \underset{1\leq r \leq b}{\max} \Big(\frac{1}{(2r + 1)^2}\sum_{(i,j) \in B_{k,\ell,r}}\hat{X}_{i,j}\Big)\right)^2\\
&= \sum_{(k,\ell)\in K_1} (\gamma_{k,\ell,r_{k,\ell}^*}-\hat{\gamma}_{k,\ell,\hat{r}_{k,\ell}^*})^2+\sum_{(k,\ell)\in K_2} (\hat{\gamma}_{k,\ell,\hat{r}_{k,\ell}^*}-\gamma_{k,\ell,r_{k,\ell}^*})^2\\
\label{equation7}
&\leq \sum_{(k,\ell)\in K_1} (\gamma_{k,\ell,r_{k,\ell}^*}-\hat{\gamma}_{k,\ell,r_{k,\ell}^*})^2 +\sum_{(k,\ell)\in K_2} (\hat{\gamma}_{k,\ell,\hat{r}_{k,\ell}^*}-\gamma_{k,l,\hat{r}_{k,\ell}^*})^2\\
\label{equation8}
&=\sum_{(k,\ell) \in K_1}\left(\frac{1}{(2r_{k,\ell}^*+1)^2}\sum_{(i,j) \in B_{k,l,r_{k,\ell}^*}}X_{i,j}-\frac{1}{(2r_{k,\ell}^*+1)^2}\sum_{(i,j) \in B_{k,\ell,r_{k,\ell}^*}}\hat{X}_{i,j}\right)^2 \nonumber \\
&\quad  + \sum_{(k,\ell) \in K_2}\left(\frac{1}{(2\hat{r}_{k,\ell}^*+1)^2}\sum_{(i,j) \in B_{k,\ell,\hat{r}_{k,\ell}^*}}\hat{X}_{i,j}-\frac{1}{(2\hat{r}_{k,\ell}^*+1)^2}\sum_{(i,j) \in B_{k,\ell,\hat{r}_{k,\ell}^*}}X_{i,j}\right)^2\\
\label{equation10}
&\leq \sum_{(k,\ell) \in K_1}\frac{1}{(2r_{k,\ell}^*+1)^2}\sum_{(i,j) \in B_{k,\ell,r_{k,\ell}^*}}(X_{i,j}-\hat{X}_{i,j})^2 \nonumber \\
& \quad +\sum_{(k,\ell) \in K_2}\frac{1}{(2\hat{r}_{k,\ell}^*+1)^2}\sum_{(i,j) \in B_{k,\ell,\hat{r}_{k,\ell}^*}}(\hat{X}_{i,j}-X_{i,j})^2\\
\label{equation11}
&\leq \frac{1}{(2r_{min}^*+1)^2}\sum_{(k,\ell) \in K_1,(i,j) \in B_{k,\ell,r_{k,\ell}^*}}(X_{i,j}-\hat{X}_{i,j})^2  \nonumber \\
&\quad +
\frac{1}{(2\hat{r}_{min}^*+1)^2}\sum_{ (k,\ell)\in K_2, (i,j)\in B_{k,\ell,\hat{r}_{k,\ell}^*}}(\hat{X}_{i,j}-X_{i,j})^2\\
\label{equation12}
&\leq \sum_{(k,\ell) \in K_1, (i,j) \in B_{k,\ell,b} }(X_{i,j}-\hat{X}_{i,j})^2+\sum_{(k,\ell) \in K_2, (i,j) \in B_{k,\ell,b} }(\hat{X}_{i,j}-X_{i,j})^2\\
\label{equation13}
&\leq \|X-\hat{X}\|_2^2.
\end{align}
The inequality (\ref{equation7}) comes from the fact that $\hat{\gamma}_{k,l,\hat{r}_{k,\ell}^*}$ is the maximum over all $r$ and thus $\hat{\gamma}_{k,\ell,\hat{r}_{k,\ell}^*}\geq \hat{\gamma}_{k,\ell,r_{k,\ell}^*}$, and similarly $\gamma_{k,\ell,r_{k,\ell}^*}\geq\gamma_{k,\ell,\hat{r}_{k,\ell}^*}$. In (\ref{equation8}), $B_{k,\ell, r_k^*}$ and $B_{k,\ell, \hat{r}_k^*}$ are the corresponding set of indices for which $\gamma_{k,\ell,r_{k,\ell}^*}$ and $\hat{\gamma}_{k,\ell,\hat{r}_{k,\ell}^*}$ are maximums across all $r$'s, respectively. The inequality (\ref{equation10}) holds based on the inequality $(\sum_{i=1}^n a_i)^2 \leq n\sum a_i^2$.
 Inequality (\ref{equation12}) follows from the fact that $r\leq b$ and stride size $s \geq 2b + 1$. $B_{k,\ell,b}$ represents the indices of the sub-square of length $2b+1$ at the initial position in the $(k,l)$th pooling region.
\end{proof}
We now prove Theorem \ref{T1}. 
\begin{proof}
By Theorem 3 in \cite{papyan2016working}, we know that for a signal $Y=X+E$, if
\begin{equation*}
\begin{split}
&1.\|\Gamma_1^*\|_{0,\infty} < \frac{1}{2}(1+\frac{1}{\mu(D_1)}) \text{ and } \|E\|_2=\|Y-D_1\Gamma^*_1\|_2 \leq \epsilon_0,\\
&2.\|\hat{\Gamma}_1\|_{0,\infty} < \frac{1}{2}(1+\frac{1}{\mu(D_1)}) \text{ and } \|Y-D_1\hat{\Gamma}_1\|_2 \leq \epsilon_0,
\end{split}
\end{equation*}
then
\begin{equation}
\|\Delta\|_1^2=\|\Gamma^*_1-\hat{\Gamma}_1\|_2^2 \leq \frac{4\epsilon_0^2}{1-(2\|\Gamma_1\|_{0,\infty}-1)\mu (D_1)}=\epsilon_1^2.
\end{equation}
Since $\|\Gamma_1^*\|_{0,\infty} \leq \lambda_1$ and $\|\hat{\Gamma}_1^*\|_{0,\infty} \leq \lambda_1$ by assumption in problem \ref{DCPP} and \ref{DCPPe}, and $\lambda$ is bounded by assumption \ref{assumption} in theorem \ref{T1}, the first parts of 1 and 2 hold. Since $\Gamma_1^*$ is the solution to the $DCPP$ problem, it must be true that $\|Y-D_1\Gamma_1^*\| \leq \epsilon_0$. $\|Y-D_1\hat{\Gamma}_1\| \leq \epsilon_0$ by assumption in problem \ref{DCPPe}. Therefore, we have $\|\Delta\|_2^2 \leq \epsilon_1^2$. And hence by Lemma \ref{lemma4.2}, we have
\begin{equation}
\|P^*_1-\hat{P}_1\|_2^2 \leq \|\Delta_1\|_2^2 \leq \epsilon_1^2.
\end{equation}
At second level, the same argument holds so that $\|\Gamma_2^*\| <\lambda_2 <\frac{1}{2}(1+\frac{1}{\mu(D_1)})$ and $\|\hat{\Gamma}_2\|<\lambda_2 <\frac{1}{2}(1+\frac{1}{\mu(D_1)})$. $\|P^*_1-D_2\Gamma_2^*\|_2<\epsilon_2$ by assumption in problem \ref{DCPP}. $\|\hat{P}_1-D_2\hat{\Gamma}_2\|_2<\epsilon_2$ by assumption in problem \ref{DCPPe}. Therefore, by theorem 3 in \cite{papyan2016working} we have
\begin{equation}
\|\Gamma_2^*-\hat{\Gamma}_2\|_2^2 \leq \epsilon_2^2,
\end{equation}
and by Lemma 1, we get
\begin{equation}
\|P_2^*-\hat{P}_2\|_2^2 \leq \|\Gamma_2^*-\hat{\Gamma}_2\|_2^2 \leq \epsilon_2^2.
\end{equation}
Following this argument for $1\leq i \leq L$, we complete the proof and showed that
\begin{equation}
\|P_i^*-\hat{P}_i\|_2^2 \leq \|\Gamma_i^*-\hat{\Gamma}_i\|_2^2 \leq \epsilon_i^2, \quad \forall \quad 1 \leq i \leq L.
\end{equation}
\end{proof}

\section{Classification Experiments} \label{sec:classification}


As described in Section~\ref{sec:introduction}, pooling as a technique for dimension reduction arises in a variety of signal processing contexts.
Accordingly, the sparse coding problem described in Section~\ref{sec:sparse-coding} is not the only possible application for maxfun pooling.
In this section we consider maxfun for supervised classification problems.

Prior work has examined properties of pooling operators in the context of supervised classification.
For example,~\cite{boureau2010theoretical} performed theoretical analyses and conducted empirical experiments comparing maximum and average pooling.  Their findings indicated that, depending upon the data and features at hand, either maximum or average pooling may be preferable.  
The authors also identified potential benefits of pooling methodologies that ``interpolate" between maximum and average pooling.  As demonstrated in Section~\ref{sec:maxfun-def}, maxfun pooling also has the property that it resides between maximum and average pooling in a precise sense.  
Thus it is of interest to explore how maxfun relates not only to average and maximum pooling, but also to other intermediate pooling techniques.

\subsection{Approach}

Our goal is to highlight differences between different pooling strategies; therefore we construct a classification problem where pooling plays a significant role in developing the feature representation.
Instead of the gradual dimension reduction one might observe throughout a standard feedforward deep neural network (where 2x2 or 3x3 pooling regions might be used at any given layer) here we consider an approach akin to very shallow networks followed by a single, aggressive pooling step.
The preliminary experiments presented in this section are in the same spirit as those of~\cite{boureau2010theoretical}, albeit with a different implementation.

Since natural images often manifest meaningful correlations in the spatial dimension we choose to focus upon image classification problems.
We generate raw features by extracting the outputs from the first layer of a standard CNN (prior to pooling) that has been pre-trained on natural images.  
We then apply various pooling techniques along the spatial dimensions of these feature maps.
Finally, the vectorized feature maps are used in conjunction with a traditional multi-class support vector machine (SVM)~\cite{hearst1998support} to compare the utility of the different pooling methods.

For our experiments we use a subset of the Caltech-101 data set~\cite{fei2007learning}. To mitigate the impact of class imbalance, we limit our study to classes that have between 80 and 130 instances.  The result is an 18-class classification problem with mild class imbalance.
To standardize the spatial dimensions all images are first zero padded to make them square; e.g. a 100x120 pixel image is padded to produce a 120x120 pixel image. 
Images are kept centered when padding (in the previous example 10 rows would be added to the top and 10 to the bottom of the image).
Finally, we resize all images to 128x128 pixels.
This provides a collection of natural images with the same dimensions where some attempt has been made to preserve the aspect ratio.

For raw features we use the outputs from the first layer of the Inception-v3 network~\cite{szegedy2016rethinking}.  This deviates from~\cite{boureau2010theoretical}, which used SIFT to construct feature maps (note their study predated a number of key developments in convolutional neural networks).
We then spatially decompose the resulting feature tensor into multiple (possibly overlapping) pooling regions.
This decomposition is applied independently along each channel, and thus preserves the cardinality of the channel dimension.
After identifying the pooling regions, we apply various pooling operators. 
These pooled tensors are then vectorized and constitute the feature maps used to solve the multiclass classification problem.

\subsubsection{Baselines}
In addition to maximum and average pooling, we also compare maxfun to  other pooling strategies that sit in between these two extremes.
In particular, we consider the ``stochastic pooling" method of~\cite{zeiler2013stochastic} and the ``mixed pooling" method of~\cite{lee2016generalizing}. For stochastic pooling, the authors recommend a particular weighted average motivated by their original stochastic formulation: 
\begin{equation} \label{eq:stochastic-pooling}
	v_{k,\ell} = c_{k,\ell} \sum_{(i,j) \in Q_{k,\ell}} X_{i,j}^2, \qquad c_{k,\ell} = \frac{1}{\sum_{(i,j) \in Q_{k,\ell}} X_{i,j}} ,
\end{equation}
where $v_{k,\ell}$ denotes the pooled value and $X$ is a (non-negative) two-dimensional feature map, following the notation of Section~\ref{sec:prelim}.
Note $X$ could represent a raw image or some higher-level representation thereof.
The non-negativity assumption is consistent with image pixel intensities or with features extracted from the output of a RELU layer within a CNN.  See~\cite{zeiler2013stochastic} for full details on the stochastic origins of~\eqref{eq:stochastic-pooling}.

Mixed pooling is defined as a convex combination of maximum and average pooling; that is
\begin{equation} \label{eq:mixedpooling}
	v_{k,\ell} = \alpha \, (P_{\text{max}}X)_{k,\ell} + (1-\alpha)\, (P_{\text{avg}} X)_{k,\ell} ,
\end{equation}
where $\alpha \in (0,1)$.

The maxfun and the mixed pooling strategies both introduce hyperparameters to our experiments; in the case of the maximal function we elected to add a minimum radius $r_{min} > 1$ as a lower bound on $r$ in \eqref{eq:maxfun} while for the mixed pooling strategy we must select the scalar $\alpha$ in~\eqref{eq:mixedpooling}.
In both cases we use a $k$-fold cross-validation procedure with $k=3$ to select these hyperparameters.
Our training and testing data sets are of size 975 and 649 with the partition chosen uniformly at random.

\subsection{Preliminary Results and Discussion}
Results for our pooling experiments are summarized in table \ref{tbl:empirical}.  
Across all pooling methods we consider two scenarios: one where the pooling regions partition the spatial domain and another where the pooling windows overlap.  
For this study we implemented both centered and non-centered versions of the maxfun pooling.
The centered variant is a realization of \eqref{eq:maxfun} while the non-centered variant allows the squares to be centered at points other than the center of pooling region.

\begin{table}
	\centering
	\begin{tabular}{l|cc}
		pooling  & \multicolumn{2}{c}{SVM accuracy} \\ 
		strategy & $(b = 21, s=21)$ & $(b=21, s=11)$\\ \hline\hline
		average & 0.5763 & 0.6102\\
		maximum & 0.5932 & 0.5932 \\
		mixed & 0.6287 & 0.6240\\
		stochastic & 0.6502 & 0.6641\\
		maxfun & 0.6225 & 0.6102\\
		centered maxfun  & 0.6626 & 0.6579
	\end{tabular}
	\caption{Empirical results for our Caltech-101 classification experiments. Pooling regions have dimensions $b \times b$ and the pooling stride is $s$. Thus, when $s=b$ (column 1), the pooling regions partition the spatial domain.  When when $s<b$ (column 2) the pooling regions overlap.}
	\label{tbl:empirical}
\end{table}

Table~\ref{tbl:empirical} suggests that the centered maxfun pooling strategy provides relatively good results, on par with those of the stochastic pooling method.
Note that these results are not state-of-the-art for this problem.
The simplified network topology and aggressive pooling we utilize here are sub-optimal from a pure performance standpoint.
Recall our goal was to highlight relative differences among pooling strategies; at this stage of this work we are yet not focused on absolute performance.

It is interesting to observe the the more flexible non-centered maxfun does not appear as effective as the centered variant.  
We might speculate that edge effects are causing some issues with the non-centered variant.  
Alternatively, it is possible that the location of the pooling window may be a more significant driver in feature importance than the spatial extent of the pooling subregion.

Obviously the experiments we describe here are highly preliminary.  One direction for future work is to expand the scope these studies and to further inquire into differences between the centered vs non-centered maxfun variants.
Furthermore, there could be value in a more detailed comparison, both theoretically and empirically, of maxfun with other pooling variants that ``interpolate'' between average and maximum pooling.
Ultimately, there is an open question as to whether maxfun can benefit modern CNN architectures if it were to be included as a pooling layer.
This would entail defining a suitable notion of a gradient for discrete maxfun pooling and conducting appropriate numerical studies.
These studies would also need to address the added computational expense of maxfun over simpler methods, such as stochastic pooling~\eqref{eq:stochastic-pooling}.

\section{Conclusions}
In this paper we introduced a novel pooling method, maxfun pooling, inspired by the Hardy--Littlewood maximal function. We proved that this new pooling strategy maintains the stability of a convolutional neural network and we demonstrated its experimental performance by comparing it with state-of-the-art pooling algorithms in classification tasks.\par
Many functions and transformations originating from harmonic analysis have demonstrated the ability to extract useful features from the input data sets for classification or segmentation purposes. Maxfun pooling is another example of such a strategy that effectively extracts features from outputs of layers of neural networks and produces faithful representation of input data. 
\bibliographystyle{cjj}
\bibliography{bibitems} 
\end{document}